%% file: FLAN.tex
\documentclass[sigconf, nonacm]{acmart}
\usepackage{multirow}
\usepackage[english]{babel}
\usepackage[shortlabels]{enumitem}
\usepackage{times}
\usepackage{latexsym}
\usepackage{times}
\usepackage{latexsym}
\usepackage{graphicx}
\usepackage{enumerate}
\usepackage{amsmath,amsthm}
\usepackage{bbm}
\usepackage{booktabs}
\usepackage[utf8]{inputenc}
\usepackage[linesnumbered,ruled,vlined]{algorithm2e}
\newtheorem{proposition}{Prop.}[section]

\theoremstyle{definition}
\newtheorem{definition}{Definition}[section]

\AtBeginDocument{%
  \providecommand\BibTeX{{%
    \normalfont B\kern-0.5em{\scshape i\kern-0.25em b}\kern-0.8em\TeX}}}

\title{A Fast Randomized Algorithm for Massive Text Normalization}

\author{%
  Nan Jiang$^1$, Chen Luo$^2$, Vihan Lakshman$^2$, Yesh Dattatreya$^2$, Yexiang Xue$^1$\\
  $^1$Purdue University \\
  $^2$Amazon\\
  \texttt{\{cheluo,vihan,ydatta\}@amazon.com}\\
  \texttt{\{jiang631,yexiang\}@purdue.edu}\\
  }

\begin{document}

\renewcommand{\shortauthors}{Jiang, et al.}

\begin{abstract}
Many popular machine learning techniques in natural language processing and data mining rely heavily on high-quality text sources. However real-world text datasets contain a significant amount of spelling errors and improperly punctuated variants where the performance of these models would quickly deteriorate. Moreover, real-world, web-scale datasets contain hundreds of millions or even billions of lines of text, where the existing text cleaning tools are prohibitively expensive to execute over and may require an overhead to learn the corrections.
In this paper, we present FLAN, a scalable randomized algorithm to clean and canonicalize massive text data. Our algorithm relies on the Jaccard similarity between words to suggest correction results. We efficiently handle the pairwise word-to-word comparisons via Locality Sensitive Hashing (LSH). We also propose a novel stabilization process to address the issue of hash collisions between dissimilar words, which is a consequence of the randomized nature of LSH and is exacerbated by the massive scale of real-world datasets. 
Compared with existing approaches, our method is more efficient, both asymptotically and in empirical evaluations, and does not rely on additional features, such as lexical/phonetic similarity or word embedding features. In addition, FLAN does not require any annotated data or supervised learning.
We further theoretically show the robustness of our algorithm with upper bounds on the false positive and false negative rates of corrections. 
Our experimental results on real-world datasets demonstrate the efficiency and efficacy of FLAN. Leveraging recent advances in efficiently computing minhash signatures, FLAN requires much less computational time compared to baselines text normalization techniques on large-scale Twitter and Reddit datasets. In a human evaluation of the quality of the normalization, FLAN achieves $5\%$ and $14\%$ improvement against the baselines over the Reddit and Twitter dataset correspondingly. Our method also improves performance on the perturbed GLUE benchmark datasets, where we introduce errors into the text, and Twitter sentiment classification applications. 
\end{abstract}





\maketitle
\input{tex/intro}
\input{tex/background}
\input{tex/methodology}

\input{tex/experiment}
\input{tex/conclusion}


\bibliographystyle{ACM-Reference-Format}
\bibliography{reference}

\end{document}

%% file: tex/intro.tex
\section{Introduction}
Many Natural Language Processing (NLP) algorithms rely on high-quality text sources to obtain state-of-the-art results~\cite{10.1145/1401890.1401965,gudivada2017data}.
Recent studies have shown that model performance deteriorates when state-of-the-art models are evaluated on real-world noisy texts~\cite{DBLP:conf/aclnut/KreekA18,DBLP:journals/corr/abs-2005-00295,damaschk-etal-2019-multiclass}. 
%
%
Specifically, text data extracted from web sources such as Twitter, Reddit, and search query logs contain numerous instances of spelling errors, typos, and non-standard punctuation marks~\cite{sikdar2013cutting,DBLP:conf/emnlp/VolskePSS17}. This noise can render pretrained neural models trained on clean data sources ineffective and is challenging to clean with traditional text normalization methods on web-scale datasets.


This challenge motivates the need for \emph{lexical normalization}, which is the task of cleaning noisy input words into canonicalized forms.
Prior techniques for lexical normalization involve 1) combining similar words based on a rich set of features, such as phonetic similarity, lexical edit distances, $n$-gram probabilities, and word-embedding features~\cite{DBLP:journals/tist/HanCB13,DBLP:conf/emnlp/KajiK14}, 2) supervised learning, where annotated datasets are required to learn a correction mapping from unormalized words to normalized ones~\cite{DBLP:journals/ijdar/ChoudhurySJMSB07}, and 3) similarity search with word-embeddings, where the top-ranked words under a vector similarity measure are considered as the correction~\cite{DBLP:journals/siamcomp/ColeH02}.

In this paper, we present FLAN, a scalable randomized algorithm for lexical normalization. Compared with the existing methods, FLAN can 
1) \textit{eliminate the need for additional annotation for supervised learning}.
2) \textit{scale better on large datasets,}, especially those with hundreds of millions or billions of lines of text. 3) \textit{be robust to errors}, by reducing the likelihood of normalizing a word into a dissimilar one due to our proposed graph stabilization technique.

FLAN harnesses Locality-Sensitive Hashing (LSH)~\cite{DBLP:conf/aaai/ZhaoLM14} to find normalized words in a graph. FLAN consists of two stages, an \emph{indexing} step and an \emph{inference} step. The input to the indexing stage is the set of tokens found in the data (word unigrams in our experiments) and the output is a word-to-word directed graph, built via LSH, where all words in a connected component point to a canonicalized representative. At inference time, we use LSH again to hash an unknown word to its appropriate graph component and substitute this noisy word with the canonicalized representative from the graph. 

We further boost the probability of LSH bucketing similar words together by taking independent repetitions of the hashing process and building a weighted word-to-word graph where the weights represent the number of repetitions in which two tokens shared the same hash signature. As a stabilization step, we then remove those insignificant edges with weight below a predefined threshold. In the pruned graph, the words in every connected component are regarded as sharing the same meaning. This edge pruning operation reduces the likelihood of a word being normalized to a dissimilar one. 
We further derive upper bounds on the false positive and false negative rate of this graph construction process.

In our experiments, we compare FLAN with several popular text normalization methods over large-scale Twitter and Reddit datasets. 
In terms of running time, we find that FLAN is faster than baselines across both the indexing and inference stages. In a human evaluation on the correction quality across the Twitter and Reddit datasets, FLAN achieves a $5\%$ and $15\%$ higher F1-Score, respectively, against the competing methods. We also demonstrate the impact of FLAN on downstream NLP tasks. On the Twitter sentiment analysis and various perturbed GLUE benchmark tasks, FLAN demonstrates consistent improvement over the baselines. We also conduct an ablation study over the impact of threshold parameter on the algorithm's performance. We further provide a case study of applying FLAN in an industrial setting on the task of normalizing search queries and product titles on a dataset sampled from the search logs of a large e-commerce website. On this dataset with hundreds of millions of lines of text, we find that FLAN completed normalizing the data in a few hours while competing spell correction methods required days to finish. 
Our contributions in this paper can be summarized as follows:
\begin{itemize}
    \item We present an efficient algorithm for lexical normalization that uses the Jaccard similarity between words for lexical correction. To the best of our knowledge, this similarity measurement has not been fully explored in this domain and is different from existing word embedding search and lexical/phonetic edit distance models. Our technique does not require supervised training or annotated data. FLAN also scales better to large datasets thanks to the efficiency of LSH over competing algorithmic primitives.%
    \item  While LSH provides an efficient approach to map similar words together, its randomized nature introduces the possibility of dissimilar words sharing the same signature due to undesirable hash collisions, a problem that becomes very prevalent at massive scales. We address this challenge of dealing with unfavorable collisions through a novel approach of modeling the LSH outputs as a word-to-word graph and using multiple repetitions to identify connected components of similar entities in this graph. We show that FLAN is robust to errors and scales well to large datasets both theoretically and in our empirical experiments.
    \item We compare FLAN with several existing popular methods over different datasets, comparing the average running time, examining the quality of the word corrections via human evaluations, and providing several case studies for the performance over perturbed GLUE benchmark datasets,  Twitter sentiment analysis, and a large-scale product search dataset.
\end{itemize}

The remainder of the paper is organized as follows: Section 2 provides background information on lexical normalization and LSH; in Section 3, we present the details of the FLAN system, using LSH to hash individual words into bins and then employing a weighted word-to-word graph to determine how to convert tokens into a canonicalized representation; Section 4 presents our experimental studies showing the empirical advantages of our proposed algorithm; Section 5 recaps the contributions of this paper and also identifies some directions for future work in extending this method.

%% file: tex/background.tex
\section{Background \& Related Work}

\subsection{Lexical Normalization}
Recently, lexical normalization has received great interest with the advent of mobile computing and social networks~\cite{coddington2014correction,DBLP:conf/cikm/BonchiFNSV12}, where typing on a small keyboard increases the opportunity for typos, and the rise of social media~\cite{DBLP:conf/aclnut/BaldwinMHKRX15}, where users are accustomed to using slang, abbreviations, and other types of informal languages.
Lexical normalization refers to the process of transferring non-standard, informal, or misspelled tokens into their standardized counterparts as well as converting words of various tenses or pluralization into a consistent representation~\cite{muller2019enhancing}. This process has emerged as a crucial step to be able to utilize neural NLP models, which are often pretrained on clean text corpora, on noisy, real-world datasets.


Prior techniques in lexical normalization all involve either: 1) combining features, such as phonetic similarity and lexical distances with standard word and $n$-gram probabilities~\cite{DBLP:conf/cikm/IslamI09,DBLP:journals/tist/HanCB13,DBLP:conf/emnlp/KajiK14}, 2) supervised learning, where annotated datasets are required to learn a correction mapping~\cite{DBLP:journals/ijdar/ChoudhurySJMSB07}, or 3) nearest neighbor search within the space of word-embeddings, where the top-ranked words under the a vector similarity measure are considered as the correction candidates~\cite{DBLP:journals/siamcomp/ColeH02}.

In the literature, the classic approaches for lexical normalization usually encompass a combination of spelling correction, stemming~\cite{4410417}, and regular expression filtering~\cite{10.1145/980972.980996}. More recent works have introduced unsupervised statistical models for text cleaning~\cite{contractor2010cleansing,aw2006phrase} or combining multiple heuristics to identify and normalize out of vocabulary words~\cite{han2011lexical}. Another explored learning robust word representations through end-to-end neural networks as opposed to normalizing the data beforehand \cite{malykh-etal-2018-robust,doval2019towards} or directly fine-tuning the BERT models for lexical normalization task~\cite{muller2019enhancing}. Another group of works focus on directly learning over the subword level information, where character sequences or subword pairs are directly used for learning the representation without any correction steps~\cite{muller2019enhancing}.

However, there are several issues limiting the use of aforementioned approaches. The pattern of typos may vary across data sources and languages, possibly may require training separate supervised learning models or collecting additional labels. The current methods for lexical normalization are also either  prohibitively slow when applied over massive datasets or require expensive and time-consuming model training.

\subsection{Locality-Sensitive Hashing}
LSH is a family of functions, such that a function uniformly sampled from this family has the property that, under the hash mapping, similar points have a higher probability of having the same hash value~\cite{DBLP:conf/aaai/ZhaoLM14}. More precisely,
consider $\mathcal{H}$ a family of hash functions mapping points from $\mathbb{R}^d$ to a discrete integer set $\mathcal{U}$.

\begin{definition}[LSH Family~\cite{DBLP:conf/uai/Shrivastava015,DBLP:books/cu/LeskovecRU14}]\label{def:lsh}
A hashing family $\mathcal{H}$ is called $(R,\alpha R,p,q)$-sensitive if for any two points $x_i, x_j \in \mathbb{R}^d$  and function $h$ chosen uniformly from $\mathcal{H}$ satisfies the following properties: 
\begin{itemize}
	\item If $\text{sim}(x_i, x_j)\ge R$, then ${Pr}_{h\in\mathcal{H}}[h(x_i) = h(x_j)] \ge p$;
	\item If $\text{sim}(x_i, x_j)\le \alpha R$, then ${Pr}_{h\in\mathcal{H}}[h(x_i) = h(x_j)] \le q$.
\end{itemize}
In practice, we assume $p>q$ and $\alpha<1$.
\end{definition}

A \emph{collision} occurs when the hash values for two points are equal: $h(x_i) = h(x_j)$. The collision probability is proportional to some monotonic function of similarity between the two points: $Pr[h(x_i) = h(x_j)]$ $\propto$ $f(\text{sim}(x_i, x_j))$, where $\text{sim}(x_i, x_j)$ is the similarity under consideration and $f$ is a monotonically increasing function. Essentially, similar items are more likely to collide with each other under LSH mapping.

Minwise hashing (MinHash) is the LSH for set resemblance, also known as the Jaccard similarity~\cite{DBLP:journals/rsa/BroderM01}. The minwise hashing family applies a random permutation $\pi$ on the given set $S$, and stores only the minimum value after the permutation mapping. Given two sets $S_i, S_j$, the probability of the sets having the same MinHash value is the Jaccard  similarity between the given two sets:
\begin{equation}
Pr\left[\min \pi(S_i)=\min \pi(S_j)\right] = \frac{|S_i \cap S_j|}{|S_i \cup S_j|}    
\end{equation}

For computing several LSH signatures of the data vector, the last decade has witnessed a tremendous advance in reducing the amortized computational and memory requirements  For random projections based LSH, of which signed random projection is a special case, we can calculate $T$ LSH hashes of the data vector, with dimensions $d$, in time $O(d\log{d} + T)$, a significant improvement over $O(dT)$. This speedup is possible due to the theory of Fast-Johnson-Lindenstrauss transformation~\cite{DBLP:journals/siamcomp/AilonC09}. On the orthogonal side, even better speedup of $O(d + T)$ has been obtained with permutation-based LSH, such as minwise hashing, using ideas of densification~\cite{DBLP:conf/icml/Shrivastava17,DBLP:conf/icml/Shrivastava014,DBLP:conf/uai/Shrivastava014}. These drastic reductions in hashing time have been instrumental in making LSH based methods more appealing and practical and we leveragethese advances in our work.

In this research, we explore methods with Jaccard similarity between words. If two words have more subsequences or subwords in common, they would have a higher similarity with each other. This measure of similarity focuses solely on morphology of the words~\cite{aronoff1994morphology}, representing the structures and meanings within words. It does not incorporate the semantic or syntactic meaning, that requires the context of the words. We leave this study for future work.




%% file: tex/methodology.tex
\section{Lexical Normalization via Randomized Hashing}
\label{sec:method}

\subsection{Motivations}

To measure the distance or similarity between two words, extensive research has been conducted over two metrics: edit distance and cosine similarity. Edit distance and its variations, including Levenshtein, Damerau–Levenshtein, and Jaro-Winkler distance~\cite{cohen2003comparison} are all defined around computing the minimal sequence of edit operations (i.e., deletion, insertion, and replacement) for changing a given word into another.
Information on neighboring characters on keyboards as well as phonetic relationships are commonly applied to adjust the cost of deletion, insertion, and replacement. In the cosine similarity paradigm \cite{pmlr-v89-ding19a}, words are embedded into the Euclidean space, and the distance between two words is the angle between their corresponding word vectors. In the domain of lexical normalization, these two metrics require prohibitive computational cost when dealing with large data~\cite{DBLP:journals/tist/HanCB13,DBLP:conf/www/GuzmanL16}.
In this work, we consider Jaccard similarity as the similarity measurement between word pairs. Here, the Jaccard similarity is the ratio of character spans (or subwords) that two words share. The advantage of this metric is that it can handle web-scale data via recent algorithmic advances in computing LSH signatures~\cite{DBLP:conf/aaai/ZhaoLM14,DBLP:conf/uai/Shrivastava014}.

\subsection{Vocabulary as Lexically Similar Components}
\label{sec:flan}

\label{sec:detailed-lsh-word}
\begin{figure}[!t]
\centering
\includegraphics[width=.9\linewidth]{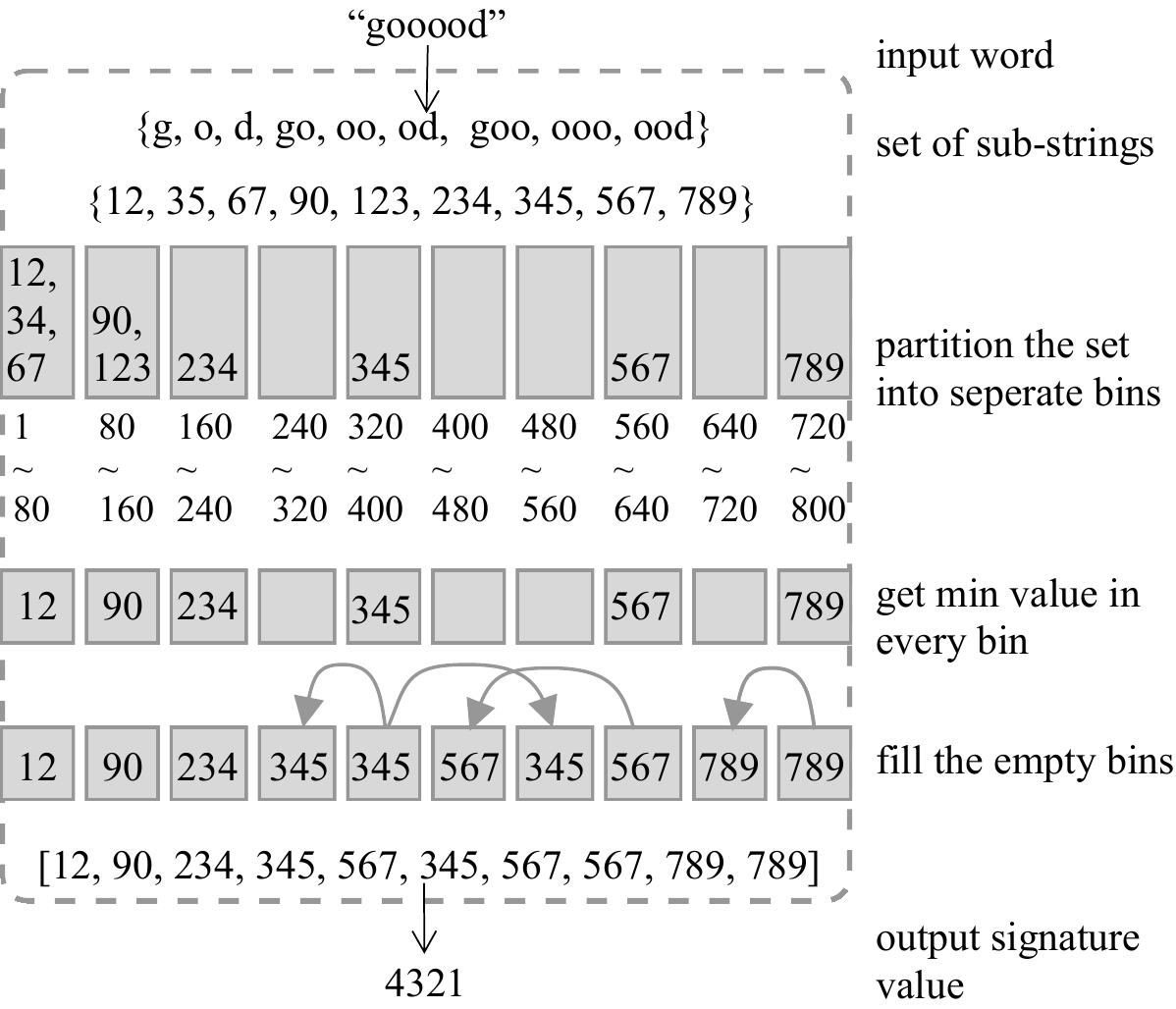}
\caption{An overview of hashing a word into a signature value via LSH. The input is a sequence of characters, sliced into substrings via the hyperparameter \texttt{CHARLENS}. Then, we use a 2-universal hash function to convert a string into a numerical (integer) value. Finally, we apply MinHash for the integer array to obtain a signature value of the input word.}
\label{fig:appendix-lsh}
\end{figure}

\noindent\textbf{Locality-Sensitive Hashing for Words.}
As shown in Figure~\ref{fig:appendix-lsh}, the detailed LSH algorithm is composed of several steps. For the first step, a word of $n$ characters $w_i=c_1c_2\cdots c_n$, is sliced into a set of substrings $\mathcal{S}(w_i)$: 
\begin{equation}
\mathcal{S}(w_i)=\{c_i\}_{i=1}^n\cup\cdots\{c_i\cdots c_{i+k}\}_{i=1}^{n-k}\cdots\cup\{c_1c_2\cdots c_n\}
\end{equation}
Here, $\mathcal{S}(w_i)$ is the union of substring sets. 
In our experiments, we introduce a hyper-parameter $\texttt{CHARLENS}$ to denote which substrings will be included in the set. For example, $\texttt{CHARLENS}=[1,3,5]$ signifies that the character-level unigram, trigram and 5-gram sets will be included into the overall set $S(w_i)$. 
If the substring length is longer than the input word length, its N-gram set is defined to be $\emptyset$. 

After obtaining the substring set, we use a hash function $h$ from a 2-universal hash family $\mathcal{H}$ to map every substring into a large universe $\mathcal{U} \subset \mathbb{N}$.

Next, we use one permutation hashing~\cite{DBLP:conf/nips/0001OZ12} to hash the output of above step. We partition the universe $U$ into bins and the set of hashed integers will be correspondingly partitioned. For example, in Figure~\ref{fig:appendix-lsh}, the universe $\mathcal{U}=\{1, 2, \dots, 800\}$ is partitioned into $10$ bins: $[1,80),\dots, [720,800]$. The integers $(12, 32, 56, 78)$ are put into the first bin $[1,80)$, and the other integers will be partitioned correspondingly. By the MinHash process, we only preserve the minimum value for those non-empty bins. For example, for the first bin, we would only preserve the minimum value of $12$. 
 
One existing issue of one permutation hashing is that we cannot have a signature for those empty bins. \citet{DBLP:conf/icml/Shrivastava17} proposes to borrow the signature value in the neighboring bins into the empty bin. In particular, for a given empty bin, we will flip a coin and borrow the first non-empty bin value from either the left or the right. This borrowing process is known as \emph{densification}. After this densification operation, we obtain an array of signature values to represent the input word $w$.

Next, we randomly hash the signature array $[s_1,\cdots,s_m]$ into an integer in another universe $\mathcal{U}'$. Here, we need another hash function $h'$ from the 2-universal hash family $\mathcal{H}$ that recursively hashes the array of signature values into one element. Each step takes the sum of the current signature value $s_i$ and the hashed value of the previous step $o_{i-1}$ as input. It will then output the hashed value for the current step: $o_i=h'(s_i+o_{i-1})$, where we use the last element value $o_m$ as the signature value for the input word $w$. We show the detailed process of mapping an input word into a signature in the universe $U'$ in Figure~\ref{fig:appendix-lsh}.

To conclude, given two words $w_i$ and $w_j$, the probability of the words having the same signature value is proportional to the Jaccard similarity of the two words. 
The probability of the event that two words $w_i,w_j$ will have the same signature value ($h(w_i)=h(w_j)$) by the LSH algorithm is proportional to their Jaccard similarity~\cite{DBLP:conf/aaai/ZhaoLM14}:
\begin{equation}
\begin{aligned}
&Pr[h(w_i)=h(w_j)]\propto \frac{|\mathcal{S}({w_i}) \cap \mathcal{S}({w_j})|}{|\mathcal{S}({w_i}) \cup \mathcal{S}({w_j})|}
\end{aligned}
\end{equation}
Here we make the assumption that all words grouped together via their signature value are lexically similar (as shown in Figure~\ref{fig:whole}). These grouped words usually are the variant of one canonical representation, which we call the \emph{representative} word. 
In this work, we use this representative to replace all of the grouped words to normalize the text data.

\begin{figure}[!t]
\centering
\includegraphics[width=.83\linewidth]{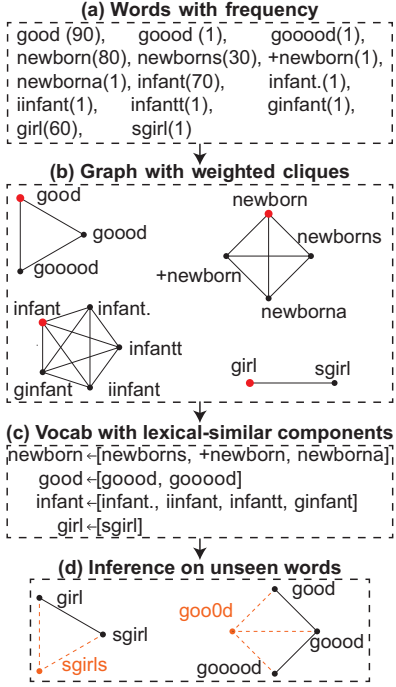}
\caption{The pipeline of FLAN algorithm. (a) Inputs are a list of words with their frequencies, which are then converted into a graph with weighted cliques via randomized hashing. (b) Afterward, we prune those in-significant edges and consider every component as words sharing the same meaning. (c) The output is a linked list style vocabulary, with all the similar words point to their pivots. (d) In inference, the new words use randomized hashing to validate if it is similar to words in vocabulary $\mathcal{V_G}$.}
\label{fig:whole}	
\end{figure}

However, due to the other property of LSH families introduced in Definition ~\ref{def:lsh}, dissimilar words can also have the same signature with an unavoidable small probability.
For real-world datasets with millions of distinct words, the chance of at least one pair of dissimilar words sharing a signature becomes significant, which leads to poor results. Motivated by the Count-Min Sketch data structure~\cite{DBLP:journals/jal/CormodeM05,luo2017arrays,chen2018}, we propose a graphical stabilization method to greatly decrease the likelihood of a word mapping to any dissimilar words while at the same time maintaining a high likelihood among similar words.

\noindent\textbf{Stabilization by Weighted Cliques.}
We stabilize the word mappings by repeating the hashing process $T$ ($T\ge 1$) times and aggregating the results using a graph data structure~\cite{luo2017arrays}. In the graph $\mathcal{G}$, let every vertex $w_i$ represent a word in the dataset such that the number of vertices in the graph is equal to the number of unique words in the dataset. We also define the edge weight $e(w_i , w_j)$ to be the number of times two words $w_i, w_j$ have the same signature value, where $e(w_i , w_j)>0$. If two words have no hash collision, then they do not have edge between them.
In one repetition, words sharing the same signature form a clique. For $T$ independent repetitions, we have a weighted graph where the weight of each edge represents the number of repetitions in which two words shared the same LSH signature. Figure~\ref{fig:whole}(b) provides an illustration of this LSH word-to-word graph. 

Furthermore, we introduce an extra criterion with a threshold parameter $\alpha\in[0,1]$ to determine if an edge weight is significant. Insignificant edges are pruned to decrease the likelihood of a word mapping to any dissimilar words. To be specific, if $e(w_i,w_j)\ge \alpha T$, then the two words $w_i,w_j$ are assumed to be sufficiently similar; if there is no edge between words $w_i$ and  $w_j$ or $e(w_i,w_j)< \alpha T$, then these two words are considered as distinct entities and we remove any edge between them in the graph. After $T$ repetitions and edge pruning, we interpret the words left in every connected component as sharing the same meaning. Note that $\alpha=0$ case means no edge will be pruned, which is simply the union of edges over all repetitions. For the $\alpha=1$ case, only words with the same signature across all the repetitions are preserved, which is the intersection for all the repetitions.

Finally, the output of the algorithm is a linked list-style vocabulary $\mathcal{V_G}$, where the lists of misspelled words are pointed to their representatives. Here we let the most frequent token in every connected component be the representative for this group of similar words. Figure~\ref{fig:whole}(c) gives one example of the output.


\noindent\textbf{Inference Criterion.}
Once we create a vocabulary $\mathcal{V_G}$ by FLAN, we can go through our input dataset and remap words to their morphological representatives. However, in the inference stage, there may exist words in the testing set that are not covered simply because they were not present in the indexing corpus. Thus, we introduce a criterion to decide if these new unseen words can be mapped back to the FLAN graph and determine which word in the vocabulary would be the best fit. 

Given the extracted vocabulary $\mathcal{V_G}$ from the indexing procedure and an unseen word $w_r$ in the testing set, we apply the LSH method for the word $w_r$ and check if the word $w_r$ would have a collision with any word in the vocabulary. After $T$ repetitions, we would have several edges that link from a set of words $ \{w_i\}_{i=1}^l$ in the vocabulary to this word $w_r$. Then we reuse our prior criterion: $e({w_r},{w_i})>\alpha T,\forall w_i\in \mathcal{V}, e(w_r,w_i)\in\mathcal{G}$, for every edge that link to $w_r$. If there are no edges ($l=0$) or none of them satisfy the criterion, this word is claimed as not similar to any words in the vocabulary. If we find more than one satisfying words, we pick the word with the largest weight. Figure~\ref{fig:whole}(d) present two examples.

\subsection{Error Analysis of Similarity Estimation}
Given a set of $N$ distinct words $\{w_i\}_{i=1}^N$ and several clusters $\{c_1,c_2,\cdots\}$, each word belongs to one and only one cluster. Let $\mathcal{C}(\cdot)$ be a mapping from a node $w_i$ to its appropriate partition. Each partition $c_k$ can be viewed as a connected component of lexically similar words.
Similar to Definition~\ref{def:lsh},
let $p$ denote the minimum probability of an intra-cluster edge and $q$ be the maximum probability of an inter-cluster edge. Let $\mathcal{G}$ denote the FLAN graph as described in the previous section. The probability of the graph having the edge $e(w_i, w_j) \in \mathcal{G}$ is:
\begin{equation}
 Pr\left[e(w_i, w_j) \in \mathcal{G}\right] \begin{cases}
\ge p \quad \text{if } \mathcal{C}(w_i)=\mathcal{C}(w_j)\\
\le q \quad \text{otherwise}
\end{cases}   
\end{equation}
In practice, we can think of $p\gg q$, as we expect the lexical similarity between words in the same component to be larger than those across the components. LSH seeks to estimate $p$ and $q$ as modeling the Jaccard similarity between words. The stabilization step is $T$ coin flips with probabilities $p$ or $q$. We first upper bound the probability of an unrelated word being included in the wrong connected component, which is the false positive probability. Then, we bound the probability that a word will not be assigned to its proper cluster by edge pruning, referred to as the false negative probability. 
\begin{proposition}[False Positive Probability]
Fix a node $w_i$, the probability that FLAN will connect $w_i$ to a node in cluster $c$ where $c \ne \mathcal{C}(w_i)$ is at most $|c|\exp\left(\frac{-T(q-\alpha)^2}{3q}\right)$ where $|c|$ is cluster $c$'s size.
\end{proposition}
\begin{proof}
Recall that $e(w_i, w_j)$ denotes the weight assigned to edge $(w_i, w_j)$. Using a union bound and a Chernoff bound, we have
\begin{equation}
\begin{aligned}
Pr[\exists w_j \in c,  e(w_i, w_j) \in \mathcal{G}] &\le \sum_{j=1}^{|c|} Pr\left[e(w_i, w_j) \ge \alpha T\right] \\
&\le |c|\exp\left(\frac{-T(q-\alpha)^2}{3q}\right)
\end{aligned}
\end{equation}
where the second inequality follows from setting $\delta = \alpha/q - 1$.
\end{proof}

\begin{proposition}[False Negative Probability]
Fix a node $w_i$, the probability that FLAN will not add an edge from $w_i$ to any of the other nodes in $c=\mathcal{C}(w_i)$ is at most $\exp\left(\frac{-|c|T (p-\alpha)^2}{2p}\right)$.
\end{proposition}
\begin{proof}
We note that $w_i$ does not share edge with some other word $w_k \in c$ in FLAN graph if the edge weight is smaller than $\alpha T$ after applying $T$ repetitions. By the fact that the presence of each edge is an independent event and another Chernoff bound, we have that
\begin{equation}
\begin{aligned}
Pr\left[\forall w_k \in c, e(w_i, w_k) \notin \mathcal{G}\right] &= \prod_{k=1}^{|c|} Pr[e(w_i, w_k) \le \alpha T] \\
&\le \exp\left(\frac{-|c|T (p-\alpha)^2}{2p}\right)
\end{aligned}
\end{equation}
where the second inequality follows from setting $\delta = 1 - \alpha/p$.
\end{proof}

The propositions imply that the probability of a false positive and false negative event decreases exponentially with more repetitions. Furthermore, we note that one meaningful difference between the two bounds is the dependence on the cluster size $|c|$. In the first case, a larger size increases the error probability while it decreases this quantity in the latter bound. 

\noindent\textbf{Running Time Analysis.}  In FLAN, the time complexity of computing the hash values is $\mathcal{O}(LNT)$, where $N$ is the number of words in the dataset, $T$ the number of repetitions of LSH, and $L$ the average number of characters in a word in the data set. Here $L$ is usually small and $L \ll N$.
Afterward, the complexity of constructing the graph is $M^2 B$, where $M$ is the expected number of items in each bucket of the hash table, and $B$ is the number of buckets.
The final graph pruning takes $\mathcal{O}(N)$ time to finish. Thus, the overall computational complexity is $\mathcal{O}(LNT + M^2B + N)$. In practice, $N$ is on the order of millions or billions and thus dominates asymptotically, so we can simplify the previous bound to $\mathcal{O}(LNT)$. Note that the major speedup of this method comes from prior breakthrough in efficiently computing the MinHash signatures~\cite{DBLP:journals/rsa/BroderM01}.

For comparison, a spell correction algorithm based on edit distance runs in time $\mathcal{O}(LNX)$, where $X$ is the number of possible characters to be deleted, replaced, or inserted. Given a word, a spell corrector will consider all the neighboring words with, for example, one and two steps of edit distances, then pick the neighboring word with maximum score in the dictionary. Usually, the number of possible characters $X$ is much larger than the number of repetitions $T$. FLAN would further improve upon the speed of edit distance-based algorithms in distributed settings where we can compute these repetitions in parallel.

\subsection{Connection to Existing Approaches}
The distance measurement used in our method is an extension and relaxation of classic stemming operations~\cite{10.5555/188490.188499}, where two words with the same stem would be of identical meaning. Our method would not only identify two words sharing the same prefix or suffix strings with high similarity, but also any subsequences of the word based on the composition of $\mathcal{S}(w)$.

The FLAN graph $\mathcal{G}$ also captures common tendencies in human errors, such as substituting adjacent characters on a keyboard or similar-sounding characters. It reduces the effort of generating features for finding patterns in typos. For a connected component of the graph, words with adjacent or similar sounding characters are included with high probability. These misspelled words are then mapped to the representative word in the final pruned graph.

Popular spell correction methods like Hunspell and Aspell\footnote{\url{http://aspell.net/}} find words that have a ``sounds like" word within a given edit distance of the original string. For FLAN, the words in a given connected component of the graph include those with small edit-distance with high probability, but this component will also likely to include words with longer edit distances, offering a dynamic and generalized way for correction. Experiment evidences for this property are collected in Table~\ref{tab: result}.

Moreover, supervised learning methods that build upon rich feature sets about human typing and spelling patterns work well on small-scale and domain-specific datasets. However, different languages and various data domains usually require adjustments, additional labeled annotations, and further feature engineering. Such expert knowledge becomes quite expensive to acquire when we scale to massive data and various languages. Our method, with no such features over typing, spelling, devices, or languages, uses multiple repetitions and pruned edge weights as statistical estimators. FLAN can effectively and efficiently normalize words to a canonical form without any supervised learning, annotations, or feature engineering.

%% file: tex/experiment.tex
\section{Experimental Study}
\label{sec:experiment}


\subsection{Experiments Setup}
\noindent\textbf{Datasets}. We consider datasets from Twitter, Reddit, the GLUE benchmark~\cite{wang-etal-2018-glue} with perturbed text, and data sampled from the logs of a large e-commerce search engine. The Twitter sentiment140 dataset contains 1.6 million of tweets with 0.7 million distinct words~\cite{DBLP:conf/comsnets/SahniCCS17}. The Reddit dataset has 10 million of sentences with 2.7 million distinct words~\cite{DBLP:conf/emnlp/VolskePSS17}. For the GLUE benchmark, we consider MRPC, STSB, RTE, CoLA and SST2 datasets, that covering single sentence prediction, sentence similarity and paraphrase along with the language inference tasks. For the e-commerce product search logs, it contains 100 million lines of product and search texts with 3.2 million unique words.
Note that the Reddit dataset is unlabeled so we only use this corpus to measure the time efficiency and correction quality of various normalization techniques and not the performance on downstream machine learning tasks. 



\noindent\textbf{Baselines}. We consider those methods with different similarity measurements for comparison: 1) edit-distance with standard word dictionary. The current popular algorithm~\cite{DBLP:journals/jdiq/Al-Hussaini17} as well as the classic method~\cite{norvig2009natural} are included. 2) cosine similarity over pretrained word-embeddings space. We use Glove~\cite{pennington2014glove} and Fasttext~\cite{mikolov-etal-2018-advances} as the word-embeddings and apply maximum inner-product search via the FAISS library for searching over the high-dimensional space~\cite{8733051}. Note that there are several lexical normalization methods are not included in this research, because either the source codes are not shared~\cite{DBLP:conf/aclnut/SupranovichP15}, the methods require annotated lexical normalization datasets~\cite{DBLP:conf/lrec/GootRCCM20}, a long pipeline with several human-defined rules are needed~\cite{DBLP:conf/www/GuzmanL16}, the methods are built upon morphological and phonetic features that are defined by domain experts~\cite{han2011lexical} or the dependencies of code were out of maintenance~\cite{DBLP:conf/acl/Goot19}.

\noindent\textbf{Evaluation Metrics.} We evaluate FLAN as well as the aforementioned lexical normalization baselines in terms of: 1) computational efficiency, which evaluates the exact running time of every algorithm, 2) correction quality, measuring the goodness of correction with human evaluators, and 3) impact on downstream applications, namely Twitter sentiment classification and perturbed GLUE benchmark datasets. 



\noindent\textbf{Hyperparameter Settings}
For the hyperparameters in FLAN, we set \texttt{CHARLENS} to be $[3, 5, 7]$. Furthermore, we set the universe size to $|\mathcal{U}|=2^{32}$ and partition the space into $4$ bins. The $2$-universal hashing function we use in our experiments is $h(x)=(ax+b)\mod P$, where $a\sim[1,|\mathcal{U}|],b\sim[0,|\mathcal{U}|-1]$ and the prime number $P=2^{31}-1$. The random seed is also fixed for reproducibility. The number of repetitions are set to $T=20$. We note that the number of repetitions determines the memory and also the running time of the FLAN algorithm. A higher number of repetitions give us a higher quality normalization while a lower number gives us a faster algorithm. Every algorithm runs over 20 cores CPU with a frequency of 3.8 GHz. We set the threshold ratio $\alpha=0.2$ for removing low weight edges. Figure~\ref{fig:phi} provide a detailed analysis for selecting this threshold.


\subsection{Correction Efficiency}
We compare the running time of all the methods over large scale datasets. For the ``Indexing'' procedure, we first extract all the words from the text corpora along with the frequencies of the words. Then, the words are fed into every  algorithm, where the output is either the original word or the corrected one. This measures the overall time to create the correction mapping for the whole training set. Only the time used for lexical normalization is calculated for these benchmarks. 
Specifically, for the ``Single'' case, the whole algorithm is applied over one process. For the ``Multi'' case, we partition the workload equally over 20 processors. For the ``Inference'' step, we benchmark the overall time for mapping words to their normalized form following the indexing stage. As shown in Table~\ref{fig:time}, we observe that FLAN has a faster running time and scales better to the dataset size than the baseline methods across both the indexing and inference stages.

\begin{table}[!ht]
    \centering
    \begin{tabular}{rrcccc}
    \toprule
   \multirow{2}{*}{Dataset} &\multirow{2}{*}{Methods} &  \multicolumn{2}{c}{Indexing (mins)} & Inference \\ 
     && Single & Multi& (mins) \\ \midrule
     \multirow{5}{*}{Twitter}&FLAN ($\alpha=0.2$) &  $\mathbf{40}\bullet$ & $\mathbf{3}\bullet$  & $\mathbf{18}\bullet$   \\
    &\citet{DBLP:journals/jdiq/Al-Hussaini17} & $171$  & $16$ & $49$ \\
     &\citet{norvig2009natural} & $510$ & $41$ & $154$ \\ 
     &FAISS-Glove & $408$ & $25$ & $83$ \\
     &FAISS-Fasttext  & $44$  & $6$ & $29$  \\ 
    \midrule
    \midrule
     \multirow{5}{*}{Reddit}&FLAN ($\alpha=0.2$) & $\mathbf{59}\bullet$  & $\mathbf{12}\bullet$& $\mathbf{26}\bullet$ \\
     &\citet{DBLP:journals/jdiq/Al-Hussaini17} & $520$ & $46$ & $71$ \\
     &\citet{norvig2009natural} & $731$ & $93$ & $221$\\ 
     &FAISS-Glove  & $514$ &$29$& $101$\\
     &FAISS-Fasttext  & $70$ & $19$& $42$ \\
    \bottomrule 
\end{tabular}
    \caption{Running time of lexical normalization methods over Twitter and Reddit Datasets.  FLAN scales better to the dataset size and is faster over Indexing and Inference scenarios than the competing approaches.}\label{fig:time}
    \vspace{-2em}
\end{table}

We acknowledge that the computation time is impacted by the choice of programming language, specific libraries, and software engineering optimizations such as caching and precomputation. The core algorithm in FAISS is implemented in C++ while the rest of the methods we benchmark are implemented in Python. 


\subsection{Correction Effectiveness}

To evaluate the quality of the corrections made by a given lexical normalization method, we conducted a study with native English speakers to evaluate the quality of the correction methods. We first select 100 sentences from both the Twitter and Reddit datasets,feed the sentences into each of our algorithms, and then extract the corrected output sentences. We create a questionnaire for the corrected sentences and deploy to the Amazon Mechanical Turk
Five different native speakers evaluated the quality of each sentence after correction. Each reviewer was asked to label every corrected sentence as either ``Good", ``Neutral'', ``Bad'', or ``Not Sure''. We define the label ``Good'' as signifying the corrections make the meaning of the text more clear or more grammatically correct. The label ``Bad'' denotes that the corrections make the meaning of the text less clear or less gramatically correct. ``Neutral'' case, signifies that the corrections do not improve or diminish the clarity of the text.

\begin{table}[!ht]
    \centering
    \begin{tabular}{rrllll}
    \toprule
Datasets&Methods&Precision& Recall & F1-Score\\ \midrule
\multirow{5}{*}{Twitter}&FLAN ($\alpha=0.2$)& $60.45\%$ & $\mathbf{41.76\%}\bullet$& $\mathbf{49.39\%}\bullet$\\
&\citet{DBLP:journals/jdiq/Al-Hussaini17}& $37.93\%$& $35.71\%$& $36.79\%$\\
&\citet{norvig2009natural}& $51.79\%$& $28.57\%$& $36.83\%$\\
&FAISS-Glove& $\mathbf{71.43\%}\bullet$& $9.34\%$ & $16.52\%$\\
&FAISS-Fasttext&$65.28\%$& $24.18\%$ & $35.28\%$\\
\midrule
\midrule
\multirow{5}{*}{Reddit}&FLAN ($\alpha$=$0.2$)& $\mathbf{84.85\%}\bullet$& $\mathbf{34.33\%}\bullet$ & $\mathbf{48.88\%}\bullet$\\
&\citet{DBLP:journals/jdiq/Al-Hussaini17}& $42.53\%$ & $\mathbf{34.33\%}\bullet$& $37.99\%$\\
&\citet{norvig2009natural}& $66.00\%$ & $32.84\%$ & $43.85\%$ \\
&FAISS-Glove& $63.64\%$ & $17.16\%$ & $27.04\%$    \\
&FAISS-Fasttext& $75.71\%$ & $22.39\%$ & $34.56\%$\\
  \bottomrule
    \end{tabular}
    \caption{Human evaluation for the quality of word corrections. For Twitter dataset, FLAN has a higher Recall and F1-score. For the Reddit dataset, and FLAN has a higher Precision and F1-Score value than the baselines.}
    \label{tab:human}
\end{table}

To evaluate recall, we further conduct another human evaluation for judging if the input sentences contain any spelling errors or typos that require lexical normalization. We repeat the same process as above but the reviewers need to label every input sentence as "Yes" or "No".

To assess the results from this study, we consider ``good" and ``neutral'' as a correct result and regard ``bad'' as an incorrect one. The precision is calculated as the ratio between the number of correct results to the total number of corrections. Recall is defined as the fraction of problematic sentences that are corrected to good. The F1-Score is calculated based on Precision and Recall~\cite{muller2019enhancing,DBLP:conf/lrec/GootRCCM20}. 

The results are presented in Table~\ref{tab:human}. For the Twitter dataset, we observe that FLAN has the highest recall and F1 score value while the FAISS-Glove method has the highest precision score. For the Reddit dataset, FLAN has the highest precision and F1 score value compared to the baselines. However, we still observe some failure cases with FLAN, such as mapping ``evga" and ``vga'' together. Disambiguating such pairs would likely require more information on the surrounding context of a word. We defer this investigation for future work.

\subsection{Impact to Downstream Applications}
\noindent\textbf{Twitter Sentiment Analysis:} We evaluate the impact of lexical normalization over real-world noisy tweets. The task is to classify the sentiment of a given tweet as positive and negative. For the neural learning model, we use the summation of word vectors as the sentence representation, which is then mapped to a two-dimensional output vector via an affine transformation. The learning objective is to minimize the logistic loss between the predicted label and the ground truth label. The word vectors inside the model are randomly initialized and we set the dimension to $256$. Prior to training the model, we apply the various lexical normalization techniques we study in our eperiments. We report the accuracy on the testing set, which we also normalize, when we reach the best result on the corresponding validation set. 

\begin{table}[!ht]
    \centering
    \begin{tabular}{rll}
    \toprule
    Methods & Valid Accuracy & Test Accuracy\\  \midrule
        No Correction &  $79.44\%$ & $79.41\%$ \\
         FLAN ($\alpha=0.2$) & $\mathbf{79.54\%} \bullet$  &  $\mathbf{79.62\%} \bullet$  \\
        \citet{DBLP:journals/jdiq/Al-Hussaini17} & $79.08\%$  & $79.16\%$  \\
        \citet{norvig2009natural} & $79.06\%$ & $79.18\%$ \\
        FAISS + Glove & $79.42\%$ & $79.41\%$  \\
        FAISS + Fasttext & $79.44\%$  &  $79.41\%$  \\
    \bottomrule
    \end{tabular}
    \caption{Accuracy results on the Twitter Dataset. The FLAN improve the Accuracy on the validation set by $0.1\%$ and testing set by $0.2\%$ against all the baselines.}
    \label{tab:classifier}
\end{table}

As shown in Table~\ref{tab:classifier}, we observe that \citet{DBLP:journals/jdiq/Al-Hussaini17,norvig2009natural} do not improve the classification result, because of the large percentage of mismatch between the language style on Twitter and formal writing. FLAN does not introduce such a domain mismatch. \\ 

\noindent\textbf{Perturbed GLUE Benchmark:} To further investigate the impact of lexical normalization tools over the related NLP tasks, we consider 5 subtasks of the popular GLUE benchmark~\cite{wang-etal-2018-glue}. As the GLUE datasets are of high-quality, we follow previous approaches~\cite{doval2019towards,DBLP:journals/corr/abs-2104-08420} in randomly perturbing the words in the validation and testing dataset while keeping the training set fixed. We generate synthetic lexical errors at $20, 40,$ and $60\%$ rates of noise such that we perturb a sentence with probability equal to this rate and then select 1-2 characters uniformly at random in every word of the sentence to delete or replace with another random character. Note that the synthesised typos are different from the real errors that follow a more structured distribution. We use a pretrained DistilBert model~\cite{DBLP:journals/corr/abs-1910-01108}, which we then fine-tune over the training set with 10 epochs. We then evaluate on the perturbed test sets after applying a normalization algorithm as a cleaning step. We also include a ``No correction" baseline as part of our study.

The results can be found at Table~\ref{tab:glue}. We observe that with the rate of noises become higher, the relative improvement of FLAN \textit{w.r.t.} the ``No correctoin'' the rest competing approaches become larger on all the chosen subtasks.  It shows that FLAN has better capability to recover the words and improve the quality of the sentences.

\subsection{Detailed Inspections}
\noindent\textbf{Ablation Study on Threshold $\alpha$.}
In Figure~\ref{fig:phi}, we plot the effect of the graph pruning threshold $\alpha$ on the behavior of FLAN. When $\alpha=0$ we see that FLAN corrects nearly every word in the corpus. However, when we set $\alpha=0.1$ or $\alpha=0.2$, we note that this correction coverage drops rapidly, which empirically demonstrates the exponential decay from applying more repetitions that we discussed previously. We also plot the correction coverage of our  baseline methods for references. Based on these results, we selected $\alpha=0.2$ as the pruning threshold in our experiences since it provided a balance between covering words and not introducing too much noise. 

\begin{figure}[!ht]
    \centering
    \includegraphics[width=.95\linewidth]{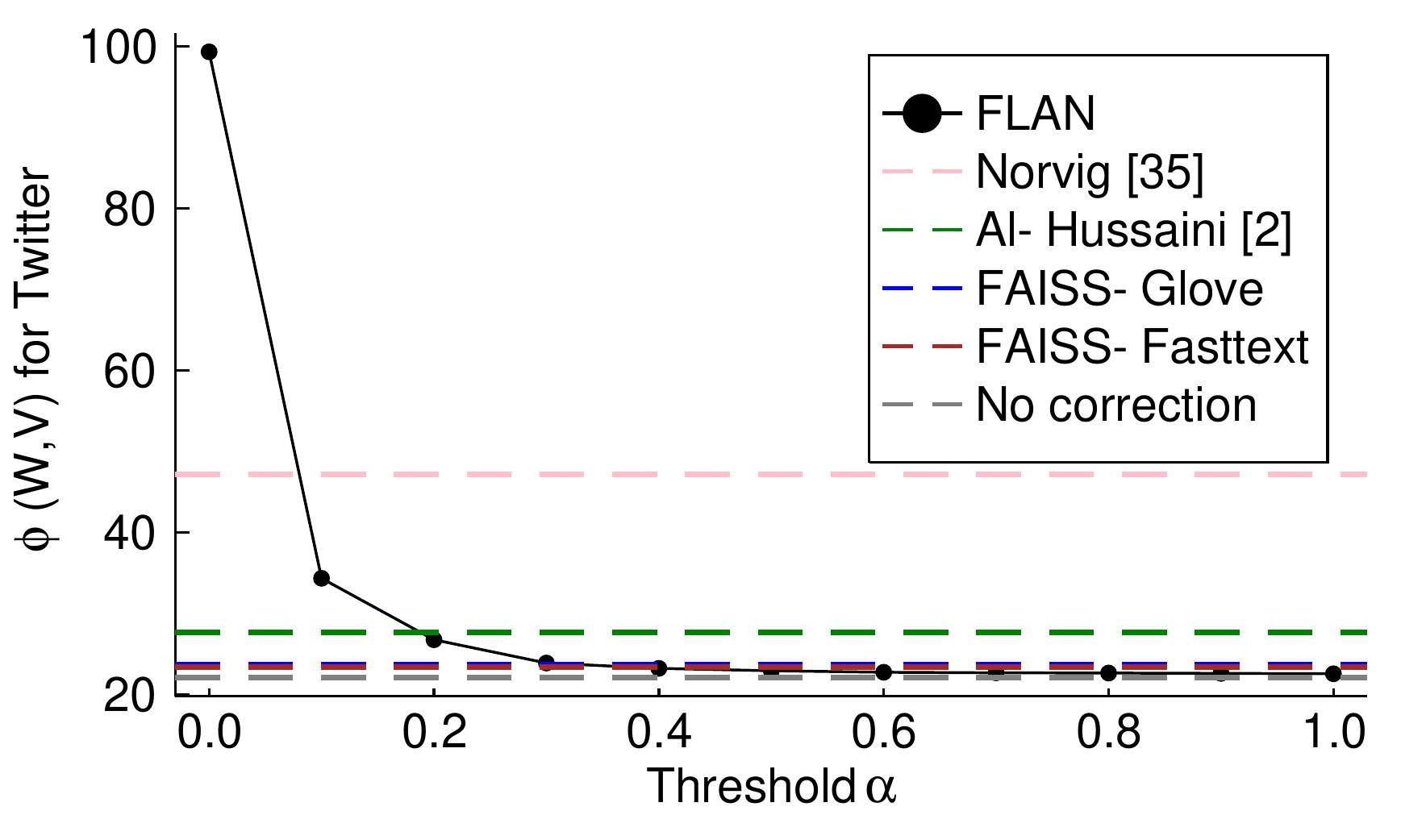}
    \includegraphics[width=.95\linewidth]{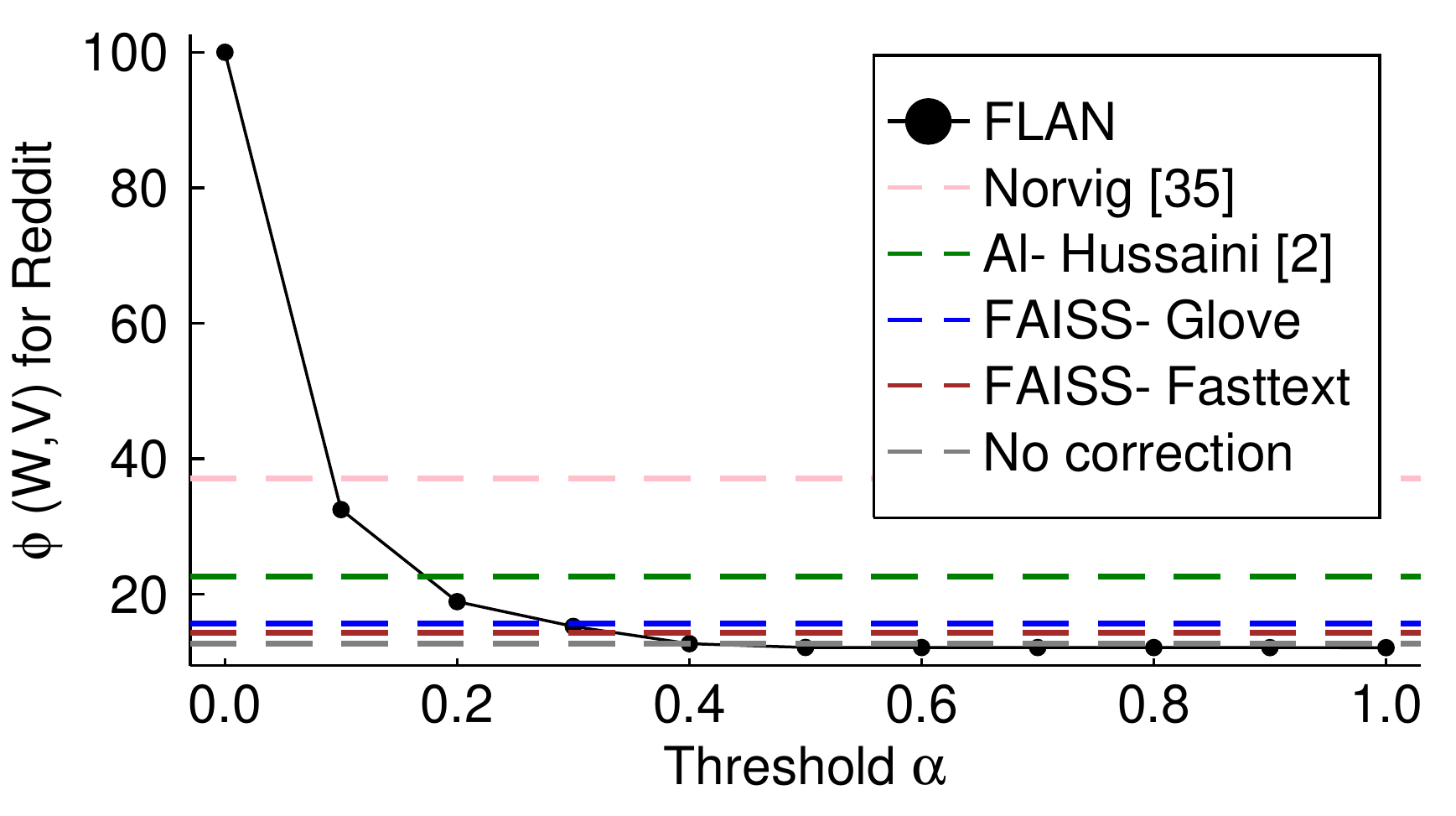}
\caption{Percentage of words get covered by the vocabulary ($\phi(\mathcal{W},\mathcal{V})$) on the Twitter and Reddit datasets. FLAN's coverage is determined by the threshold $\alpha$.}\label{fig:phi}
\end{figure}

\noindent\textbf{Case Study for Connected Components in the FLAN Graph.} We collect the results of select connected components after the LSH mapping, repetition, and pruning steps. The left column in Table~\ref{tab: result} is the representative word for the connected component while the right column illustrates other words in the  connected component that are mapped to the representative. 

\begin{table}[!ht]
\centering
\begin{tabular}{r|l}
\toprule
Representative & Connected Components \\\midrule
there&thereâ, therem, therea, ithere, therer\\\midrule
night&gnight, nightï, nightâ, gnightâ, dnight, nighti\\\midrule
friends& friend, friendsss, friendz, friendss, friendzz, \\
&friendsssss, myfriends, friendssss, vfriends, \\
&myfriend, friendâ, friend1 \\
\midrule
feeling& feelin, feelingz, feelingg, feelinga, feelinf,\\
&feelinfg \\
\midrule
morning&mornings, gmorning, morningg, gmornin, \\ 
&morningss, morningo, gmorningg, smorning, \\
&morningstar, morningâ, morningon \\
\midrule
amazing& amazingg, amazinggg, mazing, mazinggg,\\
&amazinggggg, amazinggggggg, amazingggg,\\
&amazinggggggggggg, amazingggggggggg, \\
&mazingggg, amazinggggggggg, soamazing, \\
& amazings, amazingggggg \\
\bottomrule 	
\end{tabular}
\caption{Connected Components in the constructed graph of  FLAN over Twitter Dataset. The left column is the representative word for every connected component and the right column shows other words in the corresponding connected component. We observe that FLAN can capture patterns from typographical errors on keyboards.}\label{tab: result}
\end{table}

From Table~\ref{tab: result}, we observe that FLAN can successfully group words with minor character difference into the same connected component. These results also provide evidence to the effectiveness of our graph pruning strategy in preventing spurious hash collisions from leading to unrelated word matches. This table also shows that FLAN tends to convert words of plural form into singular form or vice versa based on the frequency distribution of these variations in the dataset. In addition, FLAN is able to map infrequent words to a meaningful and frequent words in the indexed vocabulary, such as ``amazinggggg'' and ``amazingggggg''. In addition, FLAN captures typos related to the characters in close keyboard proximity such as ``feelinf'' as well as fixing the habit of double typing (``feelingg'' and ``gmorningg''). 

\input{tex/gluetable}

Ultimately, these results suggest that lexical normalization can aid in improving the quality of text-based models applied to noisy data, and FLAN provides a computationally scalable alternative to existing methods.


\subsection{Large-Scale Case Study: Product Search} 
We further conducted offline experiments applying FLAN to normalize a dataset of hundreds of millions of search queries and product titles sampled from the logs of a large e-commerce search engine~\cite{DBLP:conf/kdd/NigamSMLDSTGY19}. The structure of neural model, learning objective function and evaluation metrics follow the same settings as \citet{DBLP:conf/kdd/NigamSMLDSTGY19}. We observed that the edit-distance based methods~\cite{DBLP:journals/jdiq/Al-Hussaini17,norvig2009natural}, were prohibitively slow to apply at this scale, requiring days to complete. Meanwhile, FLAN finished normalizing the entire dataset in roughly $4$ hours. The FAISS model, on the other hand, achieved poor recall when compared to FLAN and required the additional overhead of learning these word representations on the e-commerce query-product logs.

%% file: tex/gluetable.tex
\begin{table*}[!ht]
    \centering
    \begin{tabular}{rcrcccccc}
    \toprule
 Subtask & Perturb Rate &Metrics& No corr.& FLAN ($\alpha=0.2$)& \citet{norvig2009natural} & \citet{DBLP:journals/jdiq/Al-Hussaini17}& FAISS-Glove& FAISS-Fasttext\\ \midrule
\multirow{2}{*}{MRPC}&\multirow{2}{*}{$20\%$}& Accuracy & $78.67\%$ & $\mathbf{78.92\%}\bullet$& $\mathbf{78.92\%}\bullet$& $74.26\%$ & $78.18\%$ & $78.18\%$\\
&&F1-Score & $84.26\%$ & $84.83\%$ & $84.07\%$& $82.98\%$ &$\mathbf{ 84.89\%}\bullet$ & $84.83\%$ \\ 
\midrule
\multirow{2}{*}{MRPC}&\multirow{2}{*}{$40\%$}& Accuracy  & $76.22\%$ & $\mathbf{77.94\%}\bullet$ & $77.69\%$& $74.51\%$ & $77.43\%$ & $77.69\%$ \\
&&F1-Score & $84.24\%$ & $\mathbf{85.09\%}\bullet$ & $84.17\%$& $83.38\%$ & $84.71\%$ & $84.49\%$\\
\midrule
\multirow{2}{*}{MRPC}&\multirow{2}{*}{$60\%$} & Accuracy & $67.89\%$ & $\mathbf{69.11\%}\bullet$& $67.11\%$& $65.44\%$ & $67.64\%$ & $67.89\%$\\
&&F1-Score & $74.10\%$ & $\mathbf{78.64\%}\bullet$& $74.80\%$& $72.62\%$ & $73.60\%$& $73.85\%$\\
\midrule
\midrule
\multirow{2}{*}{STSB}&\multirow{2}{*}{$20\%$}& Pearson & $72.02\%$& $\mathbf{72.55\%}\bullet$& $71.93\%$& $61.70\%$& $69.49\%$& $69.54\%$\\
&&Spearman & $71.61\%$ & $\mathbf{72.39\%}\bullet$ & $71.83\%$& $61.74\%$& $69.43\%$& $69.13\%$\\
\midrule
\multirow{2}{*}{STSB}&\multirow{2}{*}{$40\%$}& Pearson & $70.80\%$ &  $\mathbf{72.76\%\bullet}$ & $71.30\%$& $62.57\%$& $71.11\%$& $70.70\%$ \\
&&Spearman & $69.70\%$ & $\mathbf{71.67\%\bullet}$& $70.48\%$& $62.09\%$& $70.39\%$& $69.64\%$ \\
\midrule
\multirow{2}{*}{STSB}&\multirow{2}{*}{$60\%$}& Pearson & $65.73\%$&$\mathbf{70.25\%}\bullet$& $68.93\%$& $60.57\%$& $66.75\%$& $67.30\%$\\
 &&Spearman& $65.02\%$ & $\mathbf{69.90\%}\bullet$ & $68.39\%$& $60.65\%$& $66.65\%$& $66.81\%$\\
\midrule
\midrule
RTE & $20\%$  & Accuracy  & $59.57\%$ & $62.09\%$ & $58.85\%$ & $58.07\%$ & $\mathbf{61.46\%}\bullet$ & $59.21\%$  \\
RTE & $40\%$     & Accuracy &$57.76\%$ &$61.07\%$  &$57.04\%$ &$56.32\%$ &$59.57\%$ &$\mathbf{61.31\%}\bullet$\\
RTE & $60\%$ & Accuracy   &$56.32\%$ &$\mathbf{60.29\%}\bullet$ &$58.85\%$ &$54.15\%$ & $57.40\%$ &$57.40\%$ \\
\midrule
\midrule
CoLA & $20\%$  & Matthews & 	$46.00\%$ &$\mathbf{46.50\%}\bullet$& 	$42.82\%$& 	$12.25\%$& 	$39.98\%$& 	$41.03\% $\\
CoLA & $40\%$  & Matthews & 	$29.71\%$ & $\mathbf{30.92\%}\bullet$& 	$30.21\%$& 	$10.41\%$& 	$29.98\%$& 	$28.97\% $\\
CoLA & $60\%$ & Matthews & 	$9.00\%$ & $\mathbf{16.35\%}\bullet$& 	$13.44\%$& 	$15.66\%$& 	$12.86\%$& 	$15.18\% $\\
\midrule
\midrule
SST2& $20\%$ & Accuracy &$77.18\%$ & $78.72\%$&$78.93\%$&$76.76\%$ &$79.16\%$ &$\mathbf{79.23\%}\bullet$ \\
SST2& $40\%$ & Accuracy & $69.73\%$ & $\mathbf{71.23\%}\bullet$  & $70.02\%$ & $68.02\%$ & $70.74\%$ & $70.14\%$ \\
SST2& $60\%$ & Accuracy & $57.31\%$  & $\mathbf{59.42\%}\bullet$ & $58.22\%$ &$57.32\%$ & $57.67\%$  &$58.12\%$ \\
         \bottomrule
    \end{tabular}
    \caption{Preturbed GLUE benchmark with all the lexical normalization algorithms. We observe that when the noisy level become larger and larger, the FLAN can help to recover more words and get better results than all the competing methods.}
    \label{tab:glue}
\end{table*}

%% file: tex/conclusion.tex
\section{Conclusion}

In this work, we investigated lexical normalization for cleaning the real-world text data. We propose FLAN, a scalable randomized algorithm for cleaning and canonicalizing massive text data. 
By leveraging advances in randomized hashing, FLAN considerably reduces the computational complexity for large-scale text normalization. By leveraging the advance of MinHash, the approximated all word pairs are efficient computed. Compared with existing approaches, FLAN does not need extra annotation, rule definition and feature generation. 

Moreover, we propose using a graphical structure to detect and clean undesirable word associations due to random hash collisions to stabilize the correction quality. We further provide theoretical guarantees on the robustness of our algorithm with upper bounds on the false positive and false negative probabilities.

In experimental studies, we benchmark with several prevalent methods and several large-scale datasets. In running time analyses, FLAN demonstrates a faster computation speed over against methods from edit-distance models and maximum inner product search in high-dimensional word-embedding spaces. When measuring the quality of corrections, FLAN has relatively a higher recall and F1 score against the baselines as measured by human evaluation. Finally, we evaluate the end-to-end benefit of FLAN on two machine learning tasks: Twitter sentiment analysis and perturbed GLUE benchmarks, where we find that FLAN consistently improves the quality of noisy texts and help the generalization of the model.